\documentclass[twoside]{article}

\usepackage{arxiv}

\usepackage{IEEEtrantools}

\usepackage{hyperref}       % hyperlinks

\usepackage{booktabs}       % professional-quality tables
\usepackage{amsfonts}       % blackboard math symbols
\usepackage{nicefrac}       % compact symbols for 1/2, etc.
\usepackage{color}
\usepackage{microtype}      % microtypography
\usepackage{amsmath, amsthm, amssymb, amsfonts, mathtools, graphicx, enumerate}
\usepackage{caption, subcaption}
\usepackage{enumitem}
\usepackage{algorithm}
\usepackage{algorithmic}
\usepackage[round]{natbib}

\DeclareMathOperator*{\argmax}{arg\,max}

\newcommand{\R}{\mathbb{R}}

\newcommand{\C}[1]{\mathcal{#1}}

\newcommand{\E}{\mathbb{E}}
\newcommand{\Xcal}{\mathcal{X}}
\newcommand{\Scal}{\mathcal{S}}
\newcommand{\Acal}{\mathcal{A}}

\newcommand{\indi}{\mathbf{1}}
\newtheorem{theorem}{Theorem}[section]
\newtheorem{corollary}[theorem]{Corollary}

\newtheorem{assumption}[theorem]{Assumption}
\newtheorem{definition}[theorem]{Definition}
\newtheorem{example}[theorem]{Example}
\newtheorem{lemma}[theorem]{Lemma}
\newtheorem{proposition}[theorem]{Proposition}

\setlist{nolistsep}

\newcount\Comments  % 0 suppresses notes to selves in text
\definecolor{darkred}{rgb}{0.7,0,0}
\definecolor{teal}{rgb}{0.3,0.8,0.8}
\definecolor{forestgreen}{rgb}{0.13, 0.55, 0.13}
\newcommand{\kibitz}[2]{\ifnum\Comments=1{\textcolor{#1}{\textsf{\footnotesize #2}}}\fi}

\begin{document}

\twocolumn[

\aistatstitle{Sample Complexity of Reinforcement Learning using Linearly Combined Model Ensembles}

\aistatsauthor{Aditya Modi$^1$\qquad Nan Jiang$^2$\qquad Ambuj Tewari$^1$\qquad Satinder Singh$^1$}

\aistatsaddress{$^1$University of Michigan Ann Arbor\qquad $^2$University of Illinois at Urbana-Champaign} ]

\begin{abstract}
Reinforcement learning (RL) methods have %showcased their strength in 
been shown to be capable of
learning intelligent behavior in rich domains. However, this has largely been done in simulated domains without adequate focus on the process of building the simulator. In this paper, we consider a setting where we have access to an ensemble of pre-trained and possibly inaccurate simulators (models). We approximate the real environment using a state-dependent linear combination of the ensemble, where the coefficients are determined by the given state features and some unknown parameters. Our proposed algorithm provably learns a near-optimal policy with a sample complexity polynomial in the number of unknown parameters, and incurs no dependence on the size of the state (or action) space. As an extension, we also consider the more challenging problem of model selection, where the state features are unknown and can be chosen from a large candidate set. We provide exponential lower bounds that illustrate the fundamental hardness of this problem, and develop a provably efficient algorithm under additional natural assumptions.
\end{abstract}

\section{INTRODUCTION}
Reinforcement learning methods with deep neural networks as function approximators have recently demonstrated prominent success in solving complex and observations rich tasks like games \citep{mnih2015human,silver2016mastering}, simulated control problems \citep{todorov2012mujoco,lillicrap2015continuous,mordatch2016combining} and a range of robotics tasks \citep{christiano2016transfer,tobin2017domain}. A common aspect in most of these success stories is the use of simulation. % in learning behavior with the goal of maximizing long term rewards. 
Arguably, given a simulator of the real environment, it is possible to use RL to learn a near-optimal policy from (usually a large amount of) simulation data. If the simulator is highly accurate, the learned policy should also perform well in the real environment.  

Apart from some cases where the true environment and the simulator coincide (e.g., in game playing) or a nearly perfect simulator can be created from the law of physics (e.g., in simple control problems), in general we will need to construct the simulator using data from the real environment, making the overall approach an instance of \emph{model-based RL}. As the algorithms for learning from simulated experience mature (which is what the RL community has mostly focused on), the bottleneck has shifted to the creation of a good simulator. \emph{How can we learn a good model of the world from interaction experiences?}

A popular approach for meeting this challenge is to learn using a wide variety of simulators, which imparts robustness and adaptivity to the learned policies. Recent works have demonstrated the benefits of using such an ensemble of models, which can be used to either transfer policies from simulated to real-world domains, or to simply learn robust policies  \citep{andrychowicz2018learning,tobin2017domain,RajeswaranGRL17}. Borrowing the motivation from these empirical works, we notice that the process of learning a simulator inherently includes various choices like inductive biases, data collection policy, design aspects etc. As such, instead of relying on a sole approximate model for learning in simulation, interpolating between models obtained from different sources can provide better approximation of the real environment. Previous works like \cite{buckman2018sample,lee2018bayesian,kurutach2018modelensemble} have also demonstrated the effectiveness of using an ensemble of models for decreasing modelling error or its effect thereof during learning. 

In this paper, we consider building an approximate model of the real environment from interaction data using a set (or \emph{ensemble}) of possibly inaccurate models, which we will refer to as the \emph{base models}. The simplest way to combine the base models is to take a weighted combination, but such an approach is rather limited. For instance, each base model might be accurate in certain regions of the state space, in which case it is natural to consider a state-dependent mixture. We consider the problem of learning in such a setting, where one has to identify an appropriate combination of the base models through real-world interactions, so that the induced policy performs well in the real environment. The data collected through interaction with the real world can be a precious resource and, therefore, we need the learning procedure to be sample-efficient. Our main result is an algorithm  that enjoys polynomial sample complexity guarantees, where the polynomial has no dependence on the size of the state and action spaces. We also study a more challenging setting where the featurization of states for learning the combination is unknown and has to be discovered from a large set of candidate features.

\paragraph{Outline.} We formally set up the problem and notation in Section~\ref{sec:notation}, and discuss related work in Section~\ref{sec:related}. The main algorithm is introduced in Section~\ref{sec:main-algo}, together with its sample complexity guarantees. We then proceed to the feature selection problem in Section~\ref{sec:model-selec} and conclude in Section~\ref{sec:conc}.

\section{SETTING AND NOTATION}
\label{sec:notation}

We consider episodic Markov decision processes (MDP). An MDP $M$ is specified by a tuple $(\Scal, \Acal, {P}, {R}, H, P_1)$, where $\Scal$ is the state space and $\Acal$ is the action space. ${P}$ denotes the transition kernel describing the system dynamics ${P}: \Scal \times \Acal \rightarrow \Delta(\Scal)$ and ${R}$ is the per-timestep reward function ${R}: \Scal \times \Acal \rightarrow \Delta([0,1])$. The agent interacts with the environment for a fixed number of timesteps, $H$, which determines the horizon of the problem. The initial state distribution is $P_1$. The agent's goal is to find a policy $\pi: \Scal \times [H] \rightarrow \Acal$ which maximizes the value of the policy:
\begin{align*}
    v^\pi_M \coloneqq \E_{s \sim P_1} [V^\pi_{M,1}(s)] 
\end{align*}
where the value function at step $h$ is defined as: 
\begin{align*}
    V^{\pi}_{M,h}(s) = \E\Big[ \sum_{h'=h}^H r_{h'} \,\bigm\vert\, s_h=s, a_{h:H} \sim \pi,  s_{h:H} \sim M\Big]
\end{align*}
Here we use ``$s_{h:H} \sim M$'' to imply that the sequence of states are generated according to the dynamics of $M$. 
A policy is said to be optimal for $M$ if it maximizes the value $v^\pi_{M}$. We denote such a policy as $\pi_M$ and its value as $v_M$. We use $M^*$ to denote the model of the true environment, and use $\pi^*$ and $v^*$ as shorthand for $\pi_{M^*}$ and $v_{M^*}$, respectively.

In our setting, the agent is given access to a set of $K$ base MDPs $\{M_1, \ldots, M_K\}$. They share the same $\C{S}, \C{A}, H, P_1$, and only differ in ${P}$ and ${R}$. In addition, a feature map $\phi: \C{S} \times \C{A} \rightarrow \Delta_{d-1}$ is given which maps state-action pairs to $d$-dimensional real vectors. Given these two objects, we consider the class of all models which can be obtained from the following state-dependent linear combination of the base models:

\begin{definition}[Linear Combination of Base Models] \label{def:main}
For given model ensemble $\{M_1, \ldots, M_K\}$ and the feature map $\phi: \C{S} \times \C{A} \rightarrow \Delta_{d-1}$, we consider models parametrized by $W$ with the following transition and reward functions:
\begin{align*}
P^W(\cdot|s,a) = {} & \sum_{k=1}^K [W \phi(s,a)]_k P^k(\cdot|s,a), \\
R^W(\cdot| s,a) = {} & \sum_{k=1}^K [W \phi(s,a)]_k R^k(\cdot|s,a).
\end{align*}
We will use $M(W)$ to denote such a model for any parameter $W \in \C{W}$ with $\C{W}_0 \equiv \{W \in [0,1]^{K \times d}:\,\sum_{i=1}^K W_{ij} = 1 \text{ for all } j \in [d]\}$.
\end{definition}

For now, let's assume that there exists some $W^*$ such that $M^* = M(W^*)$, i.e., the true environment can be captured by our model class; we will relax this assumption shortly. 
 
To develop intuition, consider a simplified scenario where $d=1$ and $\phi(s,a) \equiv 1$. In this case, the matrix $W$ becomes a $K\times 1$ stochastic vector, and the true environment is approximated by a linear combination of the base models.  

\begin{example}[Global convex combination of models] \label{exm:const}
If the base models are combined using a set of constant weights $w \in \Delta_{K-1}$, then this is a special case of Definition~\ref{def:main} where $d=1$ and each state's feature vector is $\phi(s,a) \equiv 1$.
\end{example}

In the more general case of $d>1$, we allow the combination weights to be a linear transformation of the features, which are $W \phi(s,a)$, and hence obtain more flexibility in choosing different combination weights in different regions of the state-action space. A special case of this more general setting is when $\phi$ corresponds to a partition of the state-action space into multiple groups, and the linear combination coefficients are constant within each group. 
\begin{example}[State space partition] \label{exm:partition}
Let $\Scal \times \Acal = \bigcup_{i\in[d]} \Xcal_i$ be a partition (i.e., $\{\Xcal_i\}$ are disjoint). Let $\phi_i(s,a) = \indi[(s,a) \in\Xcal_i]$ for all $i\in[d]$ where $\indi[\cdot]$ is the indicator function. This $\phi$ satisfies the condition that $\phi(s,a) \in \Delta_{d-1}$, and when combined with a set of base models, forms a special case of Definition~\ref{def:main}.
\end{example}

\textbf{Goal.} We consider the popular PAC learning objective: with probability at least $1-\delta$, the algorithm should output a policy $\pi$ with value $v^\pi_{M^*} \ge v^* - \epsilon$ by collecting $\text{poly}(d,K,H,1/\epsilon,\log(1/\delta))$ episodes of data. Importantly, here the sample complexity is not allowed to depend on $|\Scal|$ or $|\Acal|$. However, the assumption that $M^*$ lies the class of linear models can be limiting and, therefore, we will allow some approximation error in our setting as follows:
\begin{align}
\theta \coloneqq {} & \min_{W \in \C{W}} \sup_{(s,a) \in \Scal \times \Acal}\Big\|P^*(\cdot|s,a) - P^W(\cdot|s,a)\Big\|_1 \\
{} & + \Big\|R^*(\cdot| s,a) - R^W(\cdot| s,a)\Big\|_1
\label{eq:app_error}
\end{align}
We denote the optimal parameter attaining this value by $W^*$. The case of $\theta=0$ represents the \emph{realizable} setting where $M^* = M(W^*)$ for some $W^* \in \C{W}$. When $\theta \ne 0$, we cannot guarantee returning a policy with the value close to $v^*$, and will have to pay an additional penalty term proportional to the approximation error $\theta$, as is standard in RL theory.

\paragraph{Further Notations}
Let $\pi_W$ be a shorthand for $\pi_{M(W)}$, the optimal policy in $M(W)$. When referring to value functions and state-action distributions, we will use the superscript to specify the policy and use the subscript to specify the MDP in which the policy is evaluated. For example, we will use $V^W_{W',h}$ to denote the value of $\pi_W$ (the optimal policy for model $M(W)$) when evaluated in model $M(W')$ starting from timestep $h$. The term $d^{W}_{W',h}$ denotes the state-action distribution induced by policy $\pi_W$ at timestep $h$ in the MDP $M(W')$. Furthermore, we will write $V^W_{M^*,h}$ and $d^{W}_{M^*,h}$ when the evaluation environment is $M^*$. For conciseness, $V_{W,h}$ and $Q_{W,h}$ will denote the optimal (state- and Q-) value functions in model $M(W)$ at step $h$ (e.g., $V_{W,h}(s) \equiv V^{W}_{W,h}(s)$). The expected return of a policy $\pi$ in model $M(W)$ is defined as:
\begin{align}
    v^{\pi}_{W} = \E_{s \sim P_1} [V^\pi_{M(W),1}(s)].
\end{align}
We assume that the total reward $\sum_{h=1}^H r_h$ lies in  $[0,1]$ almost surely in all MDPs of interest and under all policies. Further, whenever used, any value function at step $H+1$ (e.g., $V^{\pi}_{W,H+1}$) evaluates to $0$ for any policy and any model.

\section{RELATED WORK} \label{sec:related}

\paragraph{MDPs with low-rank transition matrices} 
\cite{yang2019reinforcement, pmlr-v97-yang19b, jin2019provably} have recently considered structured MDPs whose transition matrices admit low-rank factorization, and the left matrix in the factorization are known to the learner as state-action features (corresponding to our $\phi$). Their environmental assumption is a special case of ours, where the transition dynamics of each base model $P^k(\cdot | s, a)$ is \emph{independent} of $s$ and $a$, i.e., each base MDP can be fully specified by a single density distribution over $\Scal$. This special case enjoys many nice properties, such as the value function of any policy is also linear in state-action features, and the linear value-function class is closed under the Bellman update operators, which are heavily exploited in their algorithms and analyses. 
In contrast, none of these properties hold under our more general setup, yet we are still able to provide sample efficiency guarantees. That said, we do note that the special case allows these recent works to obtain stronger results: their algorithms are both statistically and computationally efficient (ours is only statistically efficient), and some of these algorithms work without knowing the $K$ base distributions.\footnote{In our setting, not knowing the base models immediately leads to hardness of learning, as it is  equivalent to learning a general MDP without any prior knowledge even when $d=K=1$. This requires $\Omega(|\Scal||\Acal|)$ sample complexity \citep{azar2012sample}, which is vacuous as we are interested in solving problems with arbitrarily large state and action spaces.}

\paragraph{Contextual MDPs}
\cite{abbasi2014online, modi2018markov, modi2019contextual}  consider a setting similar to our Example~\ref{exm:const}, except that the linear combination coefficients are visible to the learner and the base models are unknown. Therefore, despite the similarity in environmental assumptions, the learning objectives and the resulting sample complexities are significantly different (e.g., their guarantees depend on $|\Scal|$ and $|\Acal|$). 

\paragraph{Bellman rank} \cite{jiang2017contextual} have identified a structural parameter called Bellman rank for exploration under general value-function approximation, and devised an algorithm called OLIVE whose sample complexity is polynomial in the Bellman rank. A related notion is the witness rank (the model-based analog of Bellman rank) proposed by \citet{pmlr-v99-sun19a}. While our algorithm and analyses draw inspiration from these works, our setting does not obviously yield low Bellman rank or witness rank.\footnote{In contrast, the low-rank MDPs considered by \cite{yang2019reinforcement, pmlr-v97-yang19b, jin2019provably} do admit low Bellman rank and low witness rank.} We will also provide a more detailed comparison to \citet{pmlr-v99-sun19a}, whose algorithm is most similar to ours among the existing works, in Section~\ref{sec:main-algo}.

\paragraph{Mixtures/ensembles of models} The closest work to our setting is the multiple model-based RL (MMRL) architecture proposed by \cite{doya2002multiple} where they also decompose a given domain as a convex combination of multiple models. However, instead of learning the combination coefficients for a given ensemble, their method trains the model ensemble and simultaneously learns a mixture weight for each \emph{base model} as a function of state features. Their experiments demonstrate that each model specialized for different domains of the state space where the environment dynamics is predictable, thereby, providing a justification for using convex combination of models for simulation. Further, the idea of combining different models is inherently present in Bayesian learning methods where a posterior approximation of the real environment is iteratively refined using interaction data. For instance, \cite{RajeswaranGRL17} introduce the EPOpt algorithm which uses an ensemble of simulated domains to learn robust and generalizable policies. During learning, they adapt the ensemble distribution (convex combination) over source domains using data from the target domain to progressively make it a better approximation. Similarly, \cite{lee2018bayesian} combine a set of parameterized models by adaptively refining the mixture distribution over the latent parameter space. Here, we study a relatively simpler setting where a finite number of such base models are combined and give a frequentist sample complexity analysis for our method.

\section{ALGORITHM AND MAIN RESULTS}
\label{sec:main-algo}

In this section we introduce the main algorithm that learns a near-optimal policy in the aforementioned setup with a $\text{poly}(d,K,H,1/\epsilon,\log(1/\delta))$ sample complexity. We will first give the intuition behind the algorithm, and then present the formal  sample complexity guarantees. Due to space constraints, we present the complete proof in the appendix. For simplicity, we will describe the intuition for the realizable case with $\theta=0$ ($P^* \equiv P^{W^*}$). The pseudocode (Algorithm~\ref{alg:explore}) and the results are, however, stated for the general case of $\theta \ne 0$.

\begin{algorithm}[!ht]
	\caption{PAC Algorithm for Linear Model Ensembles}
	\label{alg:explore}
	\begin{algorithmic}[1]
		\STATE {\bfseries Input: }$\{M_1, \ldots M_K\}, \epsilon, \delta, \phi(\cdot,\cdot), \C{W}_0$
		\FOR{$t \rightarrow 1,2,\ldots$}
		\STATE Compute \emph{optimistic model} $W_t$ and set $\pi_t$ to $\pi_{W_t}$\label{line:opt_plan}
		\begin{align*}
		 W_t \leftarrow {} & \argmax_{W \in \C{W}_{t-1}} V_W   
		\end{align*}
		\STATE Estimate the value of $\pi_t$ using $n_{\text{eval}}$ trajectories:
		\begin{align}
		\hat{v}_t \coloneqq \frac{1}{n_{\text{eval}}} \sum_{h=1}^H r_h^{(i)}
		\label{algeq:misfit-error}
		\end{align}
		\IF{$v_{W_t} - \hat{v}_t \leq 3\epsilon/4 + (3\sqrt{dK}+1)H\theta$} \label{line:termination}
		\STATE Terminate and output $\pi_t$
		\ENDIF
		\STATE Collect $n$ trajectories using $\pi_t: a_h \sim \pi_t(s_h)$ \label{line:data}
		\STATE Estimate the matrix $\widehat{Z}_{t}$ and $\hat{y}_t$ as
		\begin{align}
		\widehat{Z}_{t} \coloneqq {} & \frac{1}{n} \sum_{i=1}^n \sum_{h=1}^H \overline{V}_{t,h}\Big(s^{(i)}_{h},a^{(i)}_{h}\Big)\, \phi\Big(s^{(i)}_{h},a^{(i)}_{h}\Big)^\top \label{algeq:matrix_est} \\
		\hat{y}_t \coloneqq & {} \frac{1}{n}\sum_{i=1}^n \sum_{h=1}^H r^{(i)}_{h+1} + V_{t,h+1}\big(s^{(i)}_{h+1}\big)
		\label{algeq:true_exp}
		\end{align}
		\STATE Update the version space to $\C{W}_t$ as set:
		\begin{align}
		\label{eq:linear_cut}
		\big\{W \in \C{W}_{t-1}: \big| \hat{y}_t  - \langle W, \widehat{Z}_{t} \rangle\big| \leq \tfrac{\epsilon}{12\sqrt{dK}} + H\theta\big\}
		\end{align}
		\ENDFOR
	\end{algorithmic}
\end{algorithm}

At a high level, our algorithm proceeds in iterations $t=1, 2, \ldots$, and gradually refines a \emph{version space} $\C{W}_t$ of plausible parameters. Our algorithm follows an \emph{explore-or-terminate} template and in each iteration, either chooses to explore with a carefully chosen policy or terminates with a near-optimal policy. For exploration in the $t$-th iteration, we collect $n$ trajectories $\{(s_1^{(i)}, a_1^{(i)}, r_1^{(i)}, s_2^{(i)}, \ldots, s_H^{(i)}, a_H^{(i)}, r_H^{(i)})\}_{i\in[n]}$ following some exploration policy $\pi_t$ (Line~\ref{line:data}). A key component of the algorithm is to extract knowledge about $W^*$ from these trajectories. In particular, for every $h$, the bag of samples $\{s_{h+1}^{(i)}\}_{i\in[n]}$ may be viewed as an unbiased draw from the following distribution
\begin{align} \label{eq:true_marginal}
\frac{1}{n} \sum_{i=1}^n P^{W^*}\big(\cdot| s^{(i)}_h, a^{(i)}_h\big) . 
\end{align}
The situation for rewards is similar and will be omitted in the discussion.  
So in principle we could substitute $W^*$ in Eq.\eqref{eq:true_marginal} with any candidate $W$, and if the resulting distribution differs significantly from the real samples $\{s_{h+1}^{(i)}\}_{h\in[H], i\in[n]}$, we can assert that $W\ne W^*$ and eliminate $W$ from the version space. However, the state space can be arbitrarily large in our setting, and comparing state distributions directly can be intractable. Instead, we project the state distribution in Eq.\eqref{eq:true_marginal} using a (non-stationary) discriminator function $\{f_{t,h}\}_{h=1}^H$ (which will be chosen later) and consider the following scalar property
\begin{align} \label{eq:true_projection}
\frac{1}{n} \sum_{i=1}^n \sum_{h=1}^H  \E_{\substack{r \sim R^{W^*}\big(\cdot| s_h^{(i)}, a_h^{(i)} \big) ,\\s' \sim P^{W^*} \big(\cdot| s_h^{(i)}, a_h^{(i)}\big) }} \Big[r + f_{t,h+1}(s')\Big],
\end{align}
which can be effectively estimated by
\begin{align} \label{eq:target}
\frac{1}{n} \sum_{i=1}^n \sum_{h=1}^H \Big(r_h^{(i)} + f_{t,h+1}\big(s_{h+1}^{(i)}\big)\Big).
\end{align}
Since we have projected states onto $\mathbb{R}$, Eq.\eqref{eq:target} is the average of scalar random variables and enjoys state-space-independent concentration. Now, in order to test the validity of a parameter $W$ in a given version space, we compare the estimate in eq.~\ref{eq:target} with the prediction given by $M(W)$, which is:
\begin{align}
\frac{1}{n} \sum_{i=1}^n \sum_{h=1}^H  \E_{\substack{r \sim R^{W}\big(\cdot| s_h^{(i)}, a_h^{(i)}\big),\\s' \sim P^{W}\big(\cdot| s_h^{(i)}, a_h^{(i)}\big) }} \Big[r + f_{t,h+1}(s')\Big].
\label{eq:W_estimate}
\end{align}
As we consider a linear model class, by using linearity of expectations, Eq.\eqref{eq:W_estimate} may also be written as:
\begin{align}
{} & \frac{1}{n} \sum_{i=1}^n \sum_{h=1}^H  \Big[W\phi(s_h^{(i)},a_h^{(i)})\Big]^\top \Big[\,\overline{V}_{t,h}(s_h^{(i)},a_h^{(i)})\Big] \label{eq:update-is-low-rank} \\
= {} & \Big\langle W, \frac{1}{n} \sum_{i=1}^n \sum_{h=1}^H  \overline{V}_{t,h}(s_h^{(i)},a_h^{(i)})\, \phi(s_h^{(i)},a_h^{(i)})^\top \Big\rangle,  \label{eq:linear_measure}
%= {} & \big\langle W - W^*, Z_h \,\big\rangle
\end{align}
where $\langle A,B\rangle$ denotes $\text{Tr}(A^\top B)$ for any two matrices $A$ and $B$. In eq.~\ref{eq:linear_measure}, $\overline{V}_{t,h}$ is a function that maps $(s,a)$ to a $K$ dimensional vector with each entry being 
\begin{align}
\big[\,\overline{V}_{t,h}(s,a)\big]_k \coloneqq {} & \E_{\substack{r \sim R^{k}(\cdot| s, a),\\s' \sim P^{k}(\cdot| s, a)}} \Big[r + f_{t,h+1}(s')\Big].
\end{align}
The intuition behind  Eq.\eqref{eq:update-is-low-rank} is that for each fixed state-action pair $(s_h^{(i)}, a_h^{(i)})$, the expectation in Eq.\eqref{eq:true_projection} can be computed by first taking expectation of $r+ f_{t,h+1}(s')$ over the reward and transition distributions of each of the $K$ base models---which gives $\overline{V}_h$---and then aggregating the results using the combination coefficients. Rewriting Eq.\eqref{eq:update-is-low-rank} as Eq.\eqref{eq:linear_measure}, we see that Eq.\eqref{eq:true_projection} can also be viewed as a linear measure of $W^*$, where the measurement matrix is again $\frac{1}{n} \sum_{i=1}^n \sum_{h=1}^H  \overline{V}_{h}\big(s_h^{(i)},a_h^{(i)}\big)\, \phi\big(s_h^{(i)},a_h^{(i)}\big)^\top$. Therefore, by estimating this measurement matrix and the outcome (Eq.\eqref{eq:target}), we obtain an approximate linear equality constraint over $\C{W}_{t-1}$ and can eliminate all candidate $W$ that violates such constraints. By using a finite sample concentration bound over the inner product, we get a linear inequality constraint to update the version space (Eq.\eqref{eq:linear_cut}).  

The remaining concern is to choose the exploration policy $\pi_t$ and the discriminator function $\{f_{t,h}\}$ to ensure that the linear constraint induced in each iteration is significantly different from the previous ones and induces deep cuts in the version space. We guarantee this by choosing $\pi_t := \pi_{W_t}$ and $f_{t,h} := V_{W_t, h}$\footnote{We use the simplified notation $V_{t,h}$ for $V_{W_t,h}$.}, where $W_t$ is the \emph{optimistic model} as computed in Line~\ref{line:opt_plan}. That is, $W_t$ predicts the highest optimal value among all candidate models in $\C{W}_{t-1}$. 
Following a terminate-or-explore argument, we show that as long as $\pi_t$ is suboptimal, the linear constraint induced by our choice of $\pi_t$ and $\{f_{t,h}\}$ will significantly reduce the volume of the version space, and the iteration complexity can be bounded as poly$(d,K)$ by an ellipsoid argument similar to that of \citet{jiang2017contextual}. Similarly, the sample size needed in each iteration only depends polynomially on $d$ and $K$ and incurs no dependence on $|\Scal|$ or $|\Acal|$, as we have summarized high-dimensional objects such as $f_{t,h}$ (function over states) using low-dimensional quantities such as $\overline{V}_{t,h}$ (vector of length $K$). 

The bound on the number of iterations and the number of samples needed per iteration leads to the following sample complexity result:
\begin{theorem}[PAC bound for Alg.~\ref{alg:explore}]
\label{thm:mainbound}
In Algorithm~\ref{alg:explore}, if $n_{\text{eval}} \coloneqq \tfrac{32H^2}{\epsilon^2} \log \tfrac{4T}{\delta}$ and $n = \tfrac{1800d^2KH^2}{\epsilon^2} \log \tfrac{8dKT}{\delta}$ where $T=dK \log \tfrac{2\sqrt{2K}H}{\epsilon} / \log \tfrac{5}{3}$, with probability at least $1-
\delta$, the algorithm terminates after using at most \begin{align}
    \widetilde{\C{O}}\Big(\frac{d^3K^2H^2}{\epsilon^2} \log \frac{dKH}{\delta}\Big)
\end{align}
trajectories and returns a policy $\pi_T$ with a value $v^T \ge v^* - \epsilon-(3\sqrt{dK}+2)H\theta$.
\end{theorem}

By setting $d$ and $K$ to appropriate values, we obtain the following sample complexity bounds as corollaries: 
\begin{corollary}[Sample complexity for partitions] Since the state-action partitioning setting (Example~\ref{exm:partition}) is subsumed by the general setup, the sample complexity is again:
\begin{align}
    \widetilde{\C{O}}\Big(\frac{d^3K^2H^2}{\epsilon^2} \log \frac{dKH}{\delta}\Big)
\end{align}
\end{corollary}

\begin{corollary}[Sample complexity for global convex combination] When base models are combined without any dependence on state-action features (Example~\ref{exm:const}), the setting is special case of the general setup with $d=1$. Thus, the sample complexity is:
\begin{align}
    \widetilde{\C{O}}\Big(\frac{K^2H^2}{\epsilon^2} \log \frac{KH}{\delta}\Big)
\end{align}
\end{corollary}

Our algorithm, therefore, satisfies the requirement of learning a near-optimal policy without any dependence on the $|\Scal|$ or $|\Acal|$. Moreover, we can also account for the approximation error $\theta$ but also incur a cost of $(3\sqrt{dK}+1)H\theta$ in the performance guarantee of the final policy. As we use the projection of value functions through the linear model class, we do not model the complete dynamics of the environment. This leads to an additive loss of $3\sqrt{dK}H\theta$ in value in addition to the best achievable value loss of $2H\theta$ (see Corollary~\ref{cor:val-loss} in the appendix).

\paragraph{Comparison to OLIME \citep{pmlr-v99-sun19a}} Our Algorithm~\ref{alg:explore} shares some structural similarity with the OLIME algorithm proposed by \citet{pmlr-v99-sun19a}, but there are also several important differences. First of all, OLIME in each iteration will pick a time step and take uniformly random actions during data collection, and consequently incur polynomial dependence on $|\Acal|$ in its sample complexity. In comparison, our main data collection step (Line~\ref{line:data}) never takes a random deviation, and we do not pay any dependence on the cardinality of the action space. Secondly, similar to how we project the transition distributions onto a discriminator function (Eq.\eqref{eq:true_marginal} and \eqref{eq:true_projection}), OLIME projects the distributions onto a \emph{static discriminator class} and uses the corresponding integral probability metric (IPM) as a measure of model misfit. In our setting, however, we find that the most efficient and elegant way to extract knowledge from data is to use a \emph{dynamic} discriminator function, $V_{W_t, h}$, which changes from iteration to iteration and depends on the previously collected data. Such a choice of discriminator function allows us to make direct cuts on the parameter space $\C{W}$, whereas OLIME can only make cuts in the value prediction space.

\paragraph{Computational Characteristics} In each iteration, our algorithm computes the optimistic policy within the version space. Therefore, we rely on access to the following \emph{optimistic planning oracle}:
\begin{assumption}[Optimistic planning oracle]
We assume that when given a version space $\C{W}_t$, we can obtain the optimistic model through a single oracle call for $W_t = \argmax_{W \in \C{W}_t} V_{W}$.
\end{assumption}
It is important to note that any version space $\C{W}_t$ that we deal with is always an intersection of half-spaces induced by the linear inequality constraints. Therefore, one would hope to solve the optimistic planning problem in a computationally efficient manner given the nice geometrical form of the version space. However, even for a finite state-action space, we are not aware of any efficient solutions as the planning problem induces bilinear and non-convex constraints despite the linearity assumption. Many recently proposed algorithms also suffer from such a computational difficulty \citep{jiang2017contextual, dann2018oracle, pmlr-v99-sun19a}.

Further, we also assume that for any given $W$, we can compute the optimal policy $\pi_W$ and its value function: our elimination criteria in Eq.~\ref{eq:linear_cut} uses estimates $\widehat{Z}_t$ and $\hat{y}_t$ which in turn depend on the value function. This requirement corresponds to a standard planning oracle, and aligns with the motivation of our setting, as we can delegate these computations to any learning algorithm operating in the simulated environment with the given combination coefficient. Our algorithm, instead, focuses on careful and systematic exploration to minimize the sample complexity in the real world.

\section{MODEL SELECTION WITH CANDIDATE PARTITIONS}
\label{sec:model-selec}
In the previous section we showed that a near-optimal policy can be PAC-learned under our modeling assumptions, where the feature map $\phi: \Scal\times\Acal \to [0, 1]$ is given along with the approximation error $\theta$. In this section, we explore the more interesting and challenging setting where a realizable feature map $\phi$ is unknown, but we know that the realizable $\phi$ belongs to a candidate set $\{\phi_i\}_{i=1}^N$, i.e., the true environment satisfies our modeling assumption in Definition~\ref{def:main} under $\phi = \phi_{i^*}$ for some $i^* \in [N]$ with $\theta_{i^*} = 0$. Note that Definition~\ref{def:main} may be satisfied by multiple $\phi_i$'s; for example, adding redundant features to an already realizable $\phi_{i^*}$ still yields a realizable feature map. In such cases, we consider $\phi_{i^*}$ to be the most succinct feature map among all realizable ones, i.e., the one with the lowest dimensionality. Let $d_i$ denote the dimensionality of $\phi_i$, and $d^* = d_{i^*}$. 

One obvious baseline in this setup is to run Algorithm~\ref{alg:explore} with each $\phi_i$ and select the best policy among the returned ones. This leads to a sample complexity of roughly $\sum_{i=1}^N d_i^3$ (only the dependence on $\{d_i\}_{i=1}^N$ is considered), which can be very inefficient: When there exists $j$ such that $d^* \ll d_j$, we pay for $d_j^3$ which is much greater than the sample complexity of $d^*$; When $\{d_i\}$ are relatively uniform, we pay a linear dependence on $N$, preventing us from competing with a large set of candidate feature maps. 

So the key result we want to obtain is a sample complexity that scales as $(d^*)^3$, possibly with a mild multiplicative overhead dependence on $d^*$ and/or $N$ (e.g., $\log d^*$ and $\log N$). 

\paragraph{Hardness Result for Unstructured $\{\phi_i\}$}~\\
Unfortunately, we show that this is impossible when $\{\phi_i\}$ is unstructured via a lower bound. 
In the lower bound construction, we have an exponentially large set of candidate feature maps, all of which are state space partitions. Each of the partitions has trivial dimensionalities ($d_i=2$, $K=2$), but the sample complexity of learning is exponential, which can only be explained away as $\Omega(N)$. 

\begin{proposition} \label{prop:unstructured_feature_selection}
For the aforementioned problem of learning an $\epsilon$-optimal policy using a candidate feature set of size$N$, no algorithm can achieve $poly(d^*, K, H, 1/\epsilon, 1/\delta, N^{1-\alpha})$ sample complexity for any constant $0 < \alpha < 1$.
\end{proposition}
On a separate note, besides providing formal justification for the structural assumption we will introduce later, this proposition is of independent interest as it also sheds light on the hardness of model selection with state abstractions. We discuss the further implications in Appendix~\ref{app:abstraction}.

\begin{proof}[Proof of Proposition~\ref{prop:unstructured_feature_selection}]
We construct a linear class of MDPs with two base models $M_1$ and $M_2$ in the following way: Consider a complete tree of depth $H$ with a branching factor of $2$. The vertices forming the state space of $M_1$ and $M_2$ and the two outgoing edges in each state are the available actions. Both MDPs share the same deterministic transitions and each non-leaf node yields $0$ reward. Every leaf node yields $+1$ reward in $M_1$ and $0$ in $M_2$. Now we construct a candidate partition set $\{\phi_i\}$ of size $2^H$: for $\phi_i$, the $i$-th leaf node belongs to one equivalence class while all other leaf nodes belong to the other. (Non-leaf nodes can belong to either class as $M_1$ and $M_2$ agree on their transitions and rewards.) 

Observe that the above model class contains a finite family of $2^H$ MDPs, each of which only has 1 rewarding leaf node. Concretely, the MDP whose $i$-th leaf is rewarding is exactly realized under the feature map $\phi_i$, whose corresponding $W^*$ is the identity matrix: the $i$-th leaf yields $+1$ reward as in $M_1$, and all other leaves yield $0$ reward as in $M_2$. Learning in this family of $2^H$ MDPs is provably hard \citep{krishnamurthy2016pac}, as when the rewarding leaf is chosen adversarially, the learner has no choice but to visit almost all leaf nodes to identify the rewarding leaf as long as $\epsilon$ is below a constant threshold. The proposition follows from the fact that in this setting $d^* = 2$, $K=2$, $1/\epsilon$ is a constant, $N = 2^H$, but the sample complexity is $\Omega(2^H)$.
\end{proof}
% \begin{figure}
%     \centering
%     \includegraphics[width=0.9\columnwidth]{tree.png}
%     \caption{Tree MDP used for constructing lower bound instances. For each node, the two actions $a$ and $b$ deterministically move to the connected child node. In both Prop.~\ref{prop:unstructured_feature_selection} and Prop.~\ref{prop:unknown_vstar}, the lower bound instances encodes the optimality information in an arbitrary leaf node $i^*$ among the $2^H$ possibilities.}
%     \label{fig:tree}
% \end{figure}

This lower bound shows the necessity of introducing structural assumptions in $\{\phi_i\}$. Below, we consider a particular structure of \emph{nested partitions} that is natural and enables sample-efficient learning. Similar assumptions have also been considered in the state abstraction literature \citep[e.g.,][]{jiang2015abstraction}.

\paragraph{Nested Partitions as a Structural Assumption}~\\
Consider the case where every $\phi_i$ is a partition. W.l.o.g.~let $d_1 \le d_2 \le \ldots \le d_N$. We assume $\{\phi_i\}$ is nested, meaning that  $\forall (s,a), (s', a')$,
\begin{align*}
    \phi_i(s,a) = \phi_i(s',a') \implies \phi_j(s,a) = \phi_j(s',a'), ~~\forall i \le j.
\end{align*}

While this structural assumption almost allows us to develop sample-efficient algorithms, it is still insufficient as demonstrated by the following hardness result. 

\begin{proposition} \label{prop:unknown_vstar}
Fixing $K=2$, there exist base models $M_1$ and $M_2$ and \emph{nested} state space partitions $\phi_1$ and $\phi_2$, such that it is information-theoretically impossible for any algorithm to obtain poly$(d^*, H, K, 1/\epsilon, 1/\delta)$ sample complexity when an adversary chooses an MDP that satisfies our environmental assumption (Definition~\ref{def:main}) under either $\phi_1$ or $\phi_2$. 
\end{proposition}
\begin{proof}
We will again use an exponential tree style construction to prove the lower bound. Specifically, we construct two MDPs $M$ and $M'$ which are obtained by combining two base MDPs $M_1$ and $M_2$ using two different partitions $\phi_1$ and $\phi_2$. The specification of $M_1$ and $M_2$ is exactly the same as in the proof of Proposition~\ref{prop:unstructured_feature_selection}. We choose $\phi_1$ to be a partition of size $d_1=1$,  where all nodes are grouped together. $\phi_2$ has size $d_2=2^H$, where each leaf node belongs to a separate group. (As before, which group the inner nodes belong to does not matter.) $\phi_1$ and $\phi_2$ are obviously nested. We construct $M$ that is realizable under $\phi_2$ by randomly choosing a leaf and setting the weights for the convex combination as $(1/2+2\epsilon, 1/2-2\epsilon)$ for that leaf; for all other leafs, the weights are $(1/2,1/2)$. This is equivalent to randomly choosing $M$ from a set of $2^H$ MDPs, each of which has only one \emph{good} leaf node yielding a random reward drawn from $\text{Ber}(1/2+2\epsilon)$ instead of $\text{Ber}(1/2)$. In contrast, $M'$ is such that all leaf nodes yield $\text{Ber}(1/2)$ reward, which is realizable under $\phi_1$ with weights $(1/2, 1/2)$.

Observe that $M$ and $M'$ are exactly the same as the constructions in the proof of the multi-armed bandit lower bound by \citep{auer2002finite} (the number of arms is $2^H$), where it has been shown that distinguishing between $M$ and $M'$ takes $\Omega(2^H/\epsilon^2)$ samples. Now assume towards contradiction that there exists an algorithm that achieves poly$(d^*, H, K, 1/\epsilon, 1/\delta)$ complexity; let $f$ be the specific polynomial in its guarantee. After $f(1, H, 2, 1/\epsilon, 1/\delta)$ trajectories are collected, the algorithm must stop if the true environment is $M'$ to honor the sample complexity guarantee (since $d^* = 1$, $K=2$), and proceed to collect more trajectories if $M$ is the true environment (since $d^* = 2^H$). Making this decision essentially requires distinguishing between $M'$ and $M$ using $f(1, H, 2, 1/\epsilon, 1/\delta) = \textrm{poly}(H)$ trajectories, which contradicts the known hardness result from \cite{auer2002finite}. This proves the statement. 
\end{proof}

Essentially, the lower bound creates a situation where $d_1 \ll d_2$, and the nature may adversarially choose a model such that either $\phi_1$ or $\phi_2$ is  realizable. If $\phi_1$ is realizable, the learner is only allowed a small sample budget and cannot fully explore with $\phi_2$, and if $\phi_1$ is not realizable the learner must do the opposite. The information-theoretic lower bound shows that it is fundamentally hard to distinguish between the two situations: Once the learner explores with $\phi_1$, she cannot decide whether she should stop or move on to $\phi_2$ without collecting a large amount of data. 

This hardness result motivates our last assumption in this section, that the learner knows the value of $v^\star$ (a scalar) as side information. This way, the learner can compare the value of the returned policy in each round to $v^\star$ and effectively decide when to stop. This naturally leads to our Algorithm~\ref{alg:model-selec} that uses a doubling scheme over $\{d_i\}$, with the following sample complexity guarantee. 

\begin{algorithm}[htpb]
    \caption{Model Selection with Nested Partitions}
    \label{alg:model-selec}
    \begin{algorithmic}
    \STATE {\bfseries Input:}$\{\phi_1, \phi_2,\ldots, \phi_N\}$, $\{M_1, \ldots, M_K\}$, $\epsilon$, $\delta$, $v^\star$.
    \STATE $i \rightarrow 0$
    \WHILE{\texttt{True}}
    \STATE Choose $\phi_i$ such that $d_i$ is the largest among $\{d_j: d_j \le 2^i\}$. 
    \STATE Run Algorithm~\ref{alg:explore} on  $\Phi_i$ with $\epsilon_i = \frac{\epsilon}{2}$ and $\delta_i = \frac{\delta}{2N}$.
    \STATE Terminate the sub-routine if $t  > d_iK \log \frac{2\sqrt{2K}H}{\epsilon}/\log \frac{5}{3}$.
    \STATE Let $\pi_i$ be the returned policy (if any). Let $\hat{v}_i$ be the estimated return of $\pi_i$ using $n_{\text{eval}} = \frac{9}{2\epsilon^2}\log\tfrac{2N}{\delta}$ Monte-Carlo trajectories.
    \IF{$\hat{v}_i \ge v^* - \tfrac{2\epsilon}{3}$}
    \STATE Terminate with output $\pi_i$.
    \ENDIF
    \ENDWHILE
    \end{algorithmic}
\end{algorithm}

\begin{theorem}
When Algorithm~\ref{alg:model-selec} is run with the input $v^*$, with probability at least $1-\delta$, it returns a near-optimal policy $\pi$ with $v^{\pi} \ge v^* - \epsilon$ using at most $\tilde{\C{O}}\Big(\frac{d^{*3}K^2H^2}{\epsilon^2} \log d^* \log \frac{d^*KHN}{\delta}\Big)$ samples.
\end{theorem}
\begin{proof}
In Algorithm~\ref{alg:model-selec}, for each partition $i$, we run Algorithm~\ref{alg:explore} until termination or until the sample budget is exhausted. By union bound it is easy to verify that with probability at least $1-\delta$, all calls to Algorithm~\ref{alg:explore} will succeed and the Monte-Carlo estimation of the returned policies will be $(\epsilon/3)$-accurate, and we will only consider this success event in the rest of the proof. 
When the partition under consideration is realizable, we get $v_{M^*}^{\pi_i} \ge v^* - \epsilon/3$, therefore
\begin{align*} \textstyle
\hat{v}^{i} \ge v^{\pi_i} - \frac{\epsilon}{3} \ge v^* - \frac{2\epsilon}{3},
\end{align*}
so the algorithm will terminate after considering a realizable $\phi_i$. Similarly, whenever the algorithm terminates, we have $v^{\pi_i} \ge v^* - \epsilon$. This is because
\begin{align*} \textstyle
    v^{\pi_i} \ge \hat{v}^{i} - \frac{\epsilon}{3} \ge v^* - \epsilon,
\end{align*}
where the last inequality holds thanks to the termination condition of Algorithm~\ref{alg:model-selec}, which relies on knowledge of $v^\star$. 
The total number of iterations of the algorithm is at most $\C{O}(\log d^*)$. Therefore, by taking a union bound over all possible iterations, the sample complexity is
\begin{align*}
\sum_{i=1}^{J} \widetilde{\C{O}}\Big(\frac{d_i^3K^2H^2}{\epsilon^2} \Big) \le \widetilde{\C{O}}\Big(\frac{d^{*3}K^2H^2}{\epsilon^2} \log d^* \Big). & \qedhere
\end{align*}
\end{proof}

\paragraph{Discussion.} Model selection in online learning---especially in the context of sequential decision making---is generally considered very challenging. There has been relatively limited work in the generic setting until recently for some special cases. For instance, \cite{foster2019model} consider the model selection problem in linear contextual bandits with a sequence of nested policy classes with dimensions $d_1 < d_2 < \ldots$. They consider a similar goal of achieving sub-linear regret bounds which only scale with the optimal dimension $d_{m^*}$. In contrast to our result, they do not need to know the achievable value in the environment and give no-regret learning methods in the \emph{knowledge-free} setting. However, this is not contradictory to our lower bound: Due to the extremely delayed reward signal, our construction is equivalent to a multi-armed bandit problem with $2^H$ arms. Our negative result (Proposition~\ref{prop:unknown_vstar}) shows a lower bound on sample complexity which is exponential in horizon, therefore eliminating the possibility of sample efficient and knowledge-free model selection in MDPs.

\section{CONCLUSION}
\label{sec:conc}
In this paper, we proposed a sample efficient model based algorithms which learns a near-optimal policy by approximating the true environment via a feature dependent convex combination of a given ensemble. Our algorithm offers a sample complexity bound which is independent of the size of the environment and only depends on the number of parameters being learnt. In addition, we also consider a model selection problem, show exponential lower bounds and then give sample efficient methods under natural assumptions. The proposed algorithm and its analysis relies on a linearity assumption and shares this aspect with existing exploration methods for rich observation MDPs. We leave the possibility of considering a richer class of convex combinations to future work. Lastly, our work also revisits the open problem of coming up with a computational and sample efficient model based learning algorithm.

\subsubsection*{Acknowledgements}
This work was supported in part by a grant from the Open Philanthropy Project to the Center for Human-Compatible AI, and in part by NSF grant CAREER IIS-1452099. AT would also like to acknowledge the support of a Sloan Research Fellowship. Any opinions, findings, conclusions, or recommendations expressed here are those of the authors and do not necessarily reflect the views of the sponsors.

\bibliography{arxiv}

\onecolumn
\appendix
% !TEX root = main1.tex

\section{Proofs from the main text}
In this section, we provide a detailed proof as well as the key ideas used in the analysis. The proof uses an optimism based template which guarantees that either the algorithm terminates with a near-optimal policy or explores appropriately in the environment. We can show a polynomial sample complexity bound as the algorithm explores for a bounded number of iterations and the number of samples required in each iteration is polynomial in the desired parameters. We start with the key lemmas used in the analysis in Section~\ref{sec:keylem} with the final proof of the main theorem in Section~\ref{sec:mainproof}.

\paragraph{Notation.} As in the main text, we use $\langle X,Y\rangle$ for $\text{Tr}(X^\top Y)$. The notation $\|A\|_F$ denotes the Frobenius norm $\text{Tr}(A^\top A)$. For any matrix $A = (A^1A^2\ldots A^n)$ in $\R^{m\times n}$ with columns $A^i \in \R^m$, we will use $\|A\|_{p,q}$ as the group norm: $\|(\|A^1\|_p ,\| A^2 \|_p ,\ldots , \|A^n\|_p)\|_q$.

\subsection{Key lemmas used in the analysis}
\label{sec:keylem}
For our analysis, we first define a term $\C{E}(W,h)$ for any parameter $W$ which intuitively quantifies the model error at step $h$:
\begin{align}
	\C{E}(W,h) \coloneqq \E_{d^{W}_{M^*,h}}\big[\E_{{M(W)}} \big[r_h + V_{W,h+1}(s_{h+1}) \big| s_h,a_h\big] -  \E_{{M^*}} \big[r_h + V_{W,h+1}(s_{h+1}) \big| s_h,a_h\big]\big]
	\label{eq:per-step-error}
\end{align}

We start with the following lemma which allows us to express the value loss by using a model $M(W)$ in terms of these per-step quantities.
\begin{lemma}[Value decomposition]
\label{lem:vd}
For any $W \in \C{W}$, we can write the difference in two values:
\begin{align}
    v_W - v^{W}_{M^*} = \C{E}(W) \coloneqq \sum_{h=1}^H \C{E}(W,h)
\end{align}
\end{lemma}
\begin{proof}
We start with the value difference on the lhs:
\begin{align*}
v_W - v^{W}_{M^*} = {} & \E_{d^W_{M^*,1}} \big[\E_{M(W)} \big[r_1 + V_{W,2}(s_2) \big| s_1,a_1\big] - \E_{M^*} \big[r_1 + V^{W}_{M^*,2}(s_2) \big| s_1,a_1\big] \big]  \\
= {} & \E_{d^W_{M^*,1}} \big[\E_{M(W)} \big[r_1 + V_{W,2}(s_2) \big| s_1,a_1\big] - \E_{M^*} \big[r_1 + V_{W,2}(s_2) - V_{W,2}(s_2) + V^{W}_{M^*,2}(s_2) \big| s_1,a_1\big] \big]  \\
= {} & \E_{d^{W}_{M^*,1}}\big[\E_{{M(W)}} \big[r_1 + V_{W,2}(s_{2}) \big| s_1,a_1\big] -  \E_{{M^*}} \big[r_1 + V_{W,2}(s_{2}) \big| s_1,a_1\big]\big] \\
{} & + \E_{d^{W}_{M^*,2}}\big[ V_{W,2}(s_2) - V^W_{M^*,2}(s_2)\big]\\
= {} & \C{E}(W,1)  + \E_{d^{W}_{M^*,2}}\big[ \E_{M(W)} \big[r_2 + V_{W,3}(s_3) \big| s_2,a_2\big] - \E_{M^*} \big[r_2 + V^{W}_{M^*,3}(s_3) \big| s_2,a_2\big] \big]
\end{align*}
Unrolling the second expected value similarly till $H$ leads to the desired result.
\end{proof}

At various places in our analysis, we will use the well-known simulation lemma to compare the value of a policy $\pi$ across two MDPs:
\begin{lemma}[Simulation Lemma \citep{kearns2002near, modi2018markov}]
\label{lem:sim-lemma}
Let $M_1$ and $M_2$ be two MDPs with the same state-action space. If the transition dynamics and reward functions of the two MDPs are such that:
\begin{align*}
    \|{P}^1(\cdot|s,a) - {P}^2(\cdot|s,a)\|_1 \le {} & \epsilon_p \qquad \forall s\ \in \Scal, a \in \Acal\\
    |{R}^1(\cdot|s,a) - {R}^2(\cdot|s,a)| \le {} & \epsilon_r \qquad \forall s\ \in \Scal, a \in \Acal
\end{align*}
then, for every policy $\pi$, we have:
\begin{align}
    |v^{\pi}_{M_1} - v^{\pi}_{M_2}| \le H\epsilon_p + \epsilon_r
\end{align}
\end{lemma}

Now, we will first use the assumption about linearity to prove the following key lemma of our analysis:
\begin{lemma}[Decomposition of $\C{E}(W)$]
\label{lem:error_decomp}
If $\theta$ is the approximation error defined in eq.~\ref{eq:app_error}, then the quantity $\C{E}(W)$ can be bounded as follows:
\begin{align}
\C{E}(W) \leq \big\langle W-W^*, \sum_{h=1}^H \E_{d^{W}_{M^*,h}} \big[ \overline{V}_{W,h}(s_h,a_h) \phi(s_h,a_h)^\top \big] \big\rangle + H\theta
\end{align} 
where $\overline{V}_{W,h}(s_h,a_h) \in [0,1]^K$ is a vector with the $k^{\text{th}}$ entry as $\E_{M_k}\big[ r_h + V_{W,h+1}(s_{h+1})|s_h,a_h\big]$.
\end{lemma}
\begin{proof}
Using the definition of $\C{E}(W,h)$ from eq.~\ref{eq:per-step-error}, we rewrite the term as:
\begin{align*}
\C{E}(W) = {} & \sum_{h=1}^H \C{E}(W,h)\\
= {} & \E_{d^{W}_{M^*,h}}\big[\E_{M(W)} \big[r_h + V_{W,h+1}(s_{h+1}) \big| s_h,a_h\big] -  \E_{M^*} \big[r_h + V_{W,h+1}(s_{h+1}) \big| s_h,a_h\big]\big]\\
= {} & \sum_{h=1}^H \E_{d^{W}_{M^*,h}}\big[\E_{M(W)} \big[r_h + V_{W,h+1}(s_{h+1}) \big| s_h,a_h\big] -  \E_{M(W^*)} \big[r_h + V_{W,h+1}(s_{h+1}) \big| s_h,a_h\big]\big]\\
{} & \quad + \E_{d^{W}_{M^*,h}}\big[\E_{M(W^*)} \big[r_h + V_{W,h+1}(s_{h+1}) \big| s_h,a_h\big] -  \E_{M^*} \big[r_h + V_{W,h+1}(s_{h+1}) \big| s_h,a_h\big]\big]\\
\leq {} & \sum_{h=1}^H \E_{d^{W}_{M^*,h}}\big[\E_{M(W)} \big[r_h + V_{W,h+1}(s_{h+1}) \big| s_h,a_h\big] -  \E_{M(W^*)} \big[r_h + V_{W,h+1}(s_{h+1}) \big| s_h,a_h\big]\big] + \E_{d^{W}_{M^*,h}}\big[\theta\big]\end{align*}
Here, we rewrite the inner expectation as:
\begin{align*}
\E_{M(W)} \big[r_h + V_{W,h+1}(s_{h+1}) \big| s_h,a_h\big] = {} & \sum_{k=1}^K \big(W\phi(s_h,a_h)\big)[k] \E_{M_k}\big[r_h + V_{W,h+1}(s_{h+1}) \big| s_h,a_h\big]\\
= {} & \E_{d^{W}_{M^*,h}}\Big[ \big\langle W\phi(s_h,a_h), \overline{V}_{W,h}(s_h,a_h)\big\rangle\\
= {} & \big\langle W, \E_{d^{W}_{M^*,h}} \big[ \overline{V}_{W,h}(s_h,a_h) \phi(s_h,a_h)^\top \big]\big\rangle
\end{align*}
where $\overline{V}_{W,h}(s_h,a_h) \in [0,1]^K$ is a vector with the $k^{\text{th}}$ entry as $\E_{M_k}\big[ r_h + V_{W,h+1}(s_{h+1})|s_h,a_h\big]$. Therefore, we can finally upper bound $\C{E}(W)$ by:
\begin{align*}
\C{E}(W) \leq {} & \big\langle W - W^*, \sum_{h=1}^H \E_{d^{W}_{M^*,h}} \big[ \overline{V}_{W,h}(s_h,a_h) \phi(s_h,a_h)^\top \big]\big\rangle + H\theta
\end{align*}
\end{proof}
For conciseness, we use the notation $V_{t,h}$ for the vector $V_{W_t,h}$. We write the matrix $\sum_{h=1}^H \E_{d^{W}_{M^*,h}} \big[ \overline{V}_{W,h}(s_h,a_h) \phi(s_h,a_h)^\top$ as $Z_W$ and further use $Z_t$ for $Z_{W_t}$ which results in the bound:
\begin{align*}
\C{E}(W) \leq \big\langle W-W^*, Z_W \big\rangle + H\theta
\end{align*}

Further, using Lemma~\ref{lem:sim-lemma}, one can easily see the following result which we later use in Lemma~\ref{lem:explore-exploit}:
\begin{corollary}
\label{cor:val-loss}
For the true environment and the MDP $M(W^*)$, we have:
\begin{align}
\C{E}(W^*) \leq \big|v_{W^*} - v^{W^*}_{M^*}\big| \leq {} & H\theta
\label{eq:opt-diff}
\end{align}
\begin{align}
v^{W^*}_{M^*}  \geq {} & v^* - 2H\theta
\label{eq:val-loss}
\end{align}
\end{corollary}
\begin{proof}
Eq.~\ref{eq:opt-diff} directly follows through from the assumption and Lemma~\ref{lem:sim-lemma}. For eq.~\ref{eq:val-loss}, we have:
\begin{align*}
v^{W^*}_{M^*} \geq {} & v_{W^*} - H\theta \\
\geq {} & v^{\pi^*}_{W^*} - H\theta \\
\geq {} & v^* - 2H\theta
\end{align*}
\end{proof}

By lemma~\ref{lem:vd}, we see that if the model-misfit error is controlled at each timestep, we can directly get a bound on the value loss incurred by using the greedy policy $\pi_{W}$. In Algorithm~\ref{alg:explore}, we choose the optimistic policy $W_t$ as the exploration policy which has the following property:
\begin{lemma}[Explore-or-terminate]
\label{lem:explore-exploit}
If the estimate $\hat{v}_t$ from eq.~\ref{algeq:misfit-error} satisfies the following inequality:
\begin{align}
\label{eq:model-misfit-conc}
    \Big|\hat{v}_t - v^{W_t}_{M^*}\Big| \leq {} & \tfrac{\epsilon}{4}
\end{align}
throughout the execution of the algorithm and $W^*$ is not eliminated from any $\C{W}_t$ (version space is valid), then either of these two statements hold:
\begin{enumerate}[label=(\roman*)]
    \item the algorithm terminates with output $\pi_t$ such that $v^{\pi_t}_{W^*} \ge v^* - (3\sqrt{dK}+2)H\theta - \epsilon$
    \item the algorithm does not terminate and
    \begin{align*}
        \C{E}(W_t) \ge \frac{\epsilon}{2} + 3\sqrt{dK}H\theta + H\theta
    \end{align*}
\end{enumerate}
\end{lemma}
\begin{proof}
If the algorithm doesn't terminate, then by the condition on line~\ref{line:termination} and the assumption, we know that: 
\begin{align*}
    \C{E}(W_t) = v_{W_t} - v^{W_t}_{W^*} \ge v_{W_t} - \widehat{v}_t - \tfrac{\epsilon}{4} \ge \tfrac{\epsilon}{2} + 3\sqrt{dK}H\theta + H\theta
\end{align*}
If the algorithm does terminate at step $T$, we have:
\begin{flalign*}
&&v^{\pi_T}_{M^*} & \geq \hat{v}_t - \epsilon/4 && \text{(Eq.~\ref{eq:model-misfit-conc})}\\
&& &\geq v_{W_T} - (3\sqrt{dK}+1)H\theta - \epsilon && \text{(Alg. \ref{alg:explore}, Line \ref{line:termination})}\\
&& &\geq v_{W^*} - (3\sqrt{dK}+1)H\theta - \epsilon && \text{(Optimism)}\\
&& &\geq v^* - (3\sqrt{dK}+2)H\theta - \epsilon && \text{(Lemma~\ref{cor:val-loss})}
\end{flalign*}
\end{proof}

Lemma~\ref{lem:explore-exploit} shows that either the algorithm terminates with a (near-)optimal policy or guarantees large model-misfit error for $W_t$. For bounding the number of iterations, we use a volumetric argument similar to \cite{jiang2017contextual}. We will use the following Lemma to show the exponential rate of reduction in the volume of the version space:
\begin{lemma}[Volume reduction for MVEE, \citep{jiang2017contextual}]
\label{lem:ell}
Consider a closed and bounded set $V \subset \R^p$ and a vector $a \in \R^p$. Let $B$ be any enclosing ellipsoid of $V$ that is centered at the origin, and we abuse the same symbol for the symmetric positive definite matrix that defines the ellipsoid, i.e., $B = \{v \in \R^p : v^\top B^{-1}v \le 1\}$. Suppose there exists $u \in V$ with $|a^\top u| \ge \kappa$ and define $B_+$ as the minimum volume enclosing ellipsoid of $\{v \in B : |a^\top v| \le \gamma \}$. If $\gamma/\kappa \le 1/\sqrt{p}$, we have
\begin{align}
\label{eq:ell}
    \frac{\text{vol} (B_+)}{\text{vol} (B)} \le \sqrt{p} \frac{\gamma}{\kappa} \Big(\frac{p}{p-1}\Big)^{(p-1)/2}\Big(1-\frac{\gamma^2}{\kappa^2}\Big)^{(p-1)/2}
\end{align}
Further, if $\gamma/\kappa \leq \tfrac{1}{3\sqrt{p}}$, the RHS of eq.~\ref{eq:ell} is less than $0.6$.
\end{lemma}

In the following lemma, we now show that the exploration step can happen only a finite number of times:

\begin{lemma}[Bounding the number of iterations]
\label{lem:iteration-complexity}
If the estimates $\widehat{Z}_{t}$ and $\hat{y}_t$ in eq.~\ref{algeq:matrix_est} and \ref{algeq:true_exp} satisfy:
\begin{align}
    \Big|\hat{y}_t - \sum_{h=1}^H\E_{d^{W_t}_{M^*,h}}\big[\E_{M^*} \big[r_h + V_{t,h+1}(s_{h+1}) | s_h,a_h\big]\Big| + \Big|\langle W, \widehat{Z}_{t} \rangle - \langle W, Z_{t} \rangle\Big| \leq \frac{\epsilon}{12\sqrt{dK}}
    \label{eq:dotp-conc-linear}
\end{align}
for all $W \in \C{W}$, for all iterations in Algorithm~\ref{alg:explore}, then $W^*$ is never eliminated. Moreover, the number of exploration iterations of Algorithm~\ref{alg:explore} is at most $T = dK \log \frac{2d\sqrt{K}H}{\epsilon}/\log \frac{5}{3}$.
\end{lemma}
\begin{proof}
By definition, $W^* \in \C{W}_0$. We first show that $W^*$ is never eliminated from the version space $\C{W}_t$. Let $\alpha_t \coloneqq \sum_{h=1}^H\E_{d^{W_t}_{M^*,h}}\big[\E_{M^*} \big[r_h + V_{t,h+1}(s_{h+1}) | s_h,a_h\big]$ and by definition, we have $\langle W^*, Z_{t} \rangle \coloneqq \sum_{h=1}^H\E_{d^{W_t}_{M^*,h}}\big[\E_{M(W^*)} \big[r_h + V_{t,h+1}(s_{h+1}) | s_h,a_h\big]$. Then, we have 
\begin{align*}
\big|\hat{y}_t - \langle W^*, \widehat{Z}_{t} \rangle\big| \leq {} & \Big|\hat{y}_t - \alpha_t + \alpha_t - \langle W^*, Z_{t} \rangle + \langle W^*, Z_{t} \rangle - \langle W^*, \widehat{Z}_{t} \rangle\Big| \\
\leq {} & \frac{\epsilon}{12\sqrt{dK}} + \Big|\alpha_t -  \langle W^*, Z_{t} \rangle\Big| \\
\leq {} & \frac{\epsilon}{12\sqrt{dK}} + H\theta
\end{align*}
Therefore, $W^*$ always satisfies the update equation~\ref{eq:linear_cut} and is never eliminated.

Now, we argue that the \emph{volume} of the version space decreases at an exponential rate with each exploration iteration. To set up the volume reduction analysis, we first notice that:
\begin{align*}
    \|W-W^*\|_F \leq {} & \sqrt{\sum_{i=1}^d \big\|W^i - W^{*i}\big\|^2_2} \\
    \leq {} & \sqrt{2d}
\end{align*}
Therefore, the initial volume of a ball covering the space of flattened vectors in $\C{W}_0$ is at most $c_{dK}(\sqrt{2d})^{dK}$. We will now use Lemma~\ref{lem:ell} by considering the flattened versions of the parameter matrices $W$ in the dimension $p=dK$\footnote{For avoiding ambiguity, we will directly use the matrix notations for inner products and norms.}. Firstly, in each iteration until termination, we find a matrix $W_t$ such that $\C{E}(W_t) \ge \frac{\epsilon}{2} + (3\sqrt{dK} + 1)H\theta$. From Lemma~\ref{lem:error_decomp}, we have:
\begin{align*}
\langle W_t-W^*,Z_{t} \rangle \ge {} & \C{E}(W_t) - H\theta\\
\ge {} & \frac{\epsilon}{2} + 3\sqrt{dK}H\theta
\end{align*}
We will apply Lemma~\ref{lem:ell} with $W_t - W^*$ as the vector $u$ and $Z_t$ as the direction vector $a$. For the updated version space, $\C{W}_t$, we have: 
\begin{align*}
|\langle W-W^*,Z_{t}\rangle| = {} & |y_t - \hat{y}_t + \langle W,Z_t - \widehat{Z}_t\rangle + \widehat{Z}_t - \hat{y}_t|\\
\leq {} & \frac{\epsilon}{12\sqrt{dK}} + H\theta + \frac{\epsilon}{12\sqrt{dK}}\\
\leq {} & \frac{\epsilon}{6\sqrt{dK}} + H\theta
\end{align*}
Denoting $B_{t-1}$ as the MVEE of the version space $\C{W}_{t-1}$, we consider the MVEE $B'_t$ of the set of vectors $\C{W}'_{t} \equiv \{W \in B_{t-1}:\, |\langle W-W^*,Z_{t}\rangle| \le \epsilon/6\sqrt{dK} + H\theta \}$. Clearly, we have $\C{W}_{t-1} \subseteq B_{t-1}$, and hence, $\C{W}_{t} \subseteq \C{W}'_t$. By setting $\kappa = \tfrac{\epsilon}{2}$ and $\gamma = \tfrac{\epsilon}{3\sqrt{dK}}$ in Lemma~\ref{lem:ell}, we have: 
\begin{align*}
    \frac{\text{vol}(B_t)}{\text{vol}(B_{t-1})} \leq \frac{\text{vol}(B'_t)}{\text{vol}(B_{t-1})} \leq {} & 0.6
\end{align*}
This shows that the volume of the MVEE of the version spaces $\C{W}_t$ decreases with at least a constant rate. We now argue that the procedure stops after reaching a version space with sufficiently small volume.

For any $W \in \C{W}_t$ and $(s,a) \in \Scal \times \Acal$ we have:
\begin{align*}
    \Big\|P^W(\cdot|s,a) - P^{W^*}(\cdot|s,a)\Big\|_1 \leq {} & \Big\|\sum_{k=1}^K [(W-W^*) \phi(s,a)]_k P^k(\cdot|s,a)\Big\|_1 \\
    \leq {} & \Big\| (W-W^*)\phi(s,a) \Big\|_1 \\ 
    \leq {} & \sqrt{K} \Big\| (W-W^*)\phi(s,a) \Big\|_2 \\ 
    \leq {} & \sqrt{dK} \big\| W-W^* \big\|_F
\end{align*}
Also, using lemma~\ref{lem:sim-lemma}, if the worst case error in next state transition estimates is bounded by $\frac{\epsilon}{2H} + \theta$, the optimistic value $V^{\pi_W}_{W^*}$ is $\epsilon + H\theta$-optimal. Therefore, we only need to identify the matrix $W$ to within $\tfrac{\epsilon}{2H\sqrt{dK}}$ distance of $W^*$. Consequently, the terminating MVEE $B_T$ satisfies:
\begin{align*}
    B_T \supseteq {} & \Big\{W: \big\|W-W^* \big\|_F \leq \frac{\epsilon}{2\sqrt{dK}H}\Big\}\\
\end{align*}
Therefore, we have $\text{vol}(B_T) \ge c_{dK}(\epsilon/2\sqrt{dK}H)^{dK} $, and :
\begin{align*}
    \frac{c_{dK} (\epsilon/2\sqrt{dK}H)^{dK}}{c_{dK} (\sqrt{2d})^{dK}} \le \frac{\text{vol}(B_T)}{\text{vol}(B_0)} \le 0.6^T
\end{align*}
By solving for $T$, we get that:
\begin{align}
    dK \log \frac{2\sqrt{2K}H}{\epsilon} \ge {} & T \log \frac{5}{3} \nonumber \\
    T \le {} & dK \log \frac{2\sqrt{2K}H}{\epsilon}/\log \frac{5}{3}
\end{align}
\end{proof}

We now derive the number of trajectories required in each step to satisfy the validity requirements in Lemma~\ref{lem:explore-exploit} and \ref{lem:iteration-complexity}:
\begin{lemma}[Concentration for MC estimate $\hat{v}_t$]
\label{lem:model-misfit-conc}
For any $W_t \in \C{W}$ with probability at least $1-\delta_1$, we have:
\begin{align}
\label{eq:model-misfit-conc1}
    \Big|\hat{v}_t - v^{W_t}_{M^*}\Big| \leq \tfrac{\epsilon}{4}
\end{align}
if we set $n_{\text{eval}} \geq \tfrac{8}{\epsilon^2} \log \tfrac{2}{\delta_1}$.
\end{lemma}
\begin{proof}
Note that $\hat{v}_t$ is an unbiased estimate of $v^{W_t}_{M^*}$. From our assumption on the expected sum of rewards, the return of each trajectory is bounded by $1$ for all policies. Thus, the range of each summand for the estimate $\hat{v}_t$ is $[0,1]$. Then, the result follows from standard application of Hoeffding's inequality.
\end{proof}

\begin{lemma}[Concentration of the model misfit error]
\label{lem:dotp-conc-linear}
If $n \ge \frac{1800d^2KH^2}{\epsilon^2} \log \frac{4dK}{\delta_2}$ in Algorithm.~\ref{alg:explore}, then for a given $t$, each $W \in \C{W}$ and with probability at least $1-\delta_2$, we have:
\begin{align}
    \Big|\hat{y}_t - \sum_{h=1}^H\E_{d^{W_t}_{M^*,h}}\big[\E_{M^*} \big[r_h + V_{t,h+1}(s_{h+1}) | s_h,a_h\big]\Big| + \Big|\langle W, Z_{t} - \widehat{Z}_t \rangle\Big| \leq \frac{\epsilon}{12\sqrt{dK}}
    \label{eq:dotp-conc-linear1}
\end{align}
%As a result, with probability at least $1-\delta_2$, for any given $t$ and for all $W \in \C{W}_t$, we have:
%\begin{align}
%    \langle W-W^*, Z_{t} \rangle \le {} & \frac{\epsilon}{6\sqrt{dK}}
%\end{align}
\end{lemma}

\begin{proof}
%We firstly decompose the estimation error as follows:
%\begin{align*}
%    |y_t - \langle W, \widehat{Z}_{t} \rangle - \langle W^* - W, Z_{t} \rangle| \leq {} & |y_t - \langle W^*, Z_{t} \rangle | + | \langle W, \widehat{Z}_{t} \rangle - \langle W, Z_{t} \rangle|  & \\
%    \leq {} & |y_t - \langle W^*, Z_{t} \rangle | + \|W\|_{1,\infty} \|\widehat{Z}_{t} - Z_{t}\|_{\infty,1} &\\
%    = {} & |y_t - \langle W^*, Z_{t} \rangle | + \|\widehat{Z}_{t} - Z_{t}\|_{\infty,1}
%\end{align*}

To bound the first term, we note that $\hat{y}_t$ is an unbiased estimate of $\sum_{h=1}^H\E_{d^{W_t}_{M^*,h}}\big[\E_{M^*} \big[r_h + V_{t,h+1}(s_{h+1}) | s_h,a_h\big]$ and is bounded between $[0,H]$. Applying Hoeffding's inequality on the estimand $\hat{y}_t$ when using $n$ trajectories, with probability at least $1-\delta'$, we get:
\begin{align*}
    \Big|\hat{y}_t - \sum_{h=1}^H\E_{d^{W_t}_{M^*,h}}\big[\E_{M^*} \big[r_h + V_{t,h+1}(s_{h+1}) | s_h,a_h\big]\Big| \leq H\sqrt{\frac{1}{2n}\log \frac{2}{\delta'}}
\end{align*}

For bounding the second term, using Holder's inequality with matrix group norm \citep{Agarwal:EECS-2008-138}, we first see: 
\begin{align*}
| \langle W, \widehat{Z}_{t} \rangle - \langle W, Z_{t} \rangle| \leq {} & \|W\|_{1,\infty} \|\widehat{Z}_{t} - Z_{t}\|_{\infty,1}\\
\leq {} & \|\widehat{Z}_{t} - Z_{t}\|_{\infty,1}
\end{align*}
We will now bound the estimation error for $\widehat{Z}_t$:
\begin{align*}
    \big\|\widehat{Z}_{t} - Z_{t}\big\|_{\infty,1} = {} & \Big\|\frac{1}{n}\sum_{i=1}^n \sum_{h=1}^H \overline{V}_{t,h}(s^{(i)}_{h},a^{(i)}_{h})\, \phi(s^{(i)}_{h},a^{(i)}_{h})^\top - \E\big[\sum_{h=1}^H\overline{V}_{t,h}(s_{h},a_{h})\, \phi(s_{h},a_{h})^\top\big]\Big\|_{\infty,1} & \\
    = {} & \sum_{j=1}^d \Big\|\frac{1}{n}\sum_{i=1}^n \sum_{h=1}^H \overline{V}_{t,h}(s^{(i)}_{h},a^{(i)}_{h}) \phi(s^{(i)}_{h},a^{(i)}_{h})[j]  - \E\big[\sum_{h=1}^H \overline{V}_{t,h}(s_{h},a_{h})\, \phi(s_{h},a_{h})[j]\big]\Big\|_{\infty}&
\end{align*}
We can consider each trajectory as a random sample from the distribution of trajectories induced by $\pi_t$. Therefore, by definition, each summand in the estimate of $\widehat{Z}_t$ over $n$ trajectories is an unbiased estimate of $Z_t$. Moreover, we know that each term in the entry $\widehat{Z}_{t}[i,j]$ of the matrix $\widehat{Z}_{t}$ is bounded by $H$. Using Bernstein's inequality for each term in the error matrix, and a union bound over all entries, with probability at least $1-\delta'$, for all $i,j$ we have:
\begin{align*}
    \big|\widehat{Z}_{t}[i,j] - Z_{t}[i,j]\big| \leq {} & \sqrt{\frac{2\text{Var}\big[\,\sum_{h=1}^H \overline{V}_{t,h}[i]\phi(s_h,a_h)[j]\,\big]\log \tfrac{2dK}{\delta'}}{n}}  + \frac{2H\log \tfrac{2dK}{\delta'}}{n}&
\end{align*}

Summing up the maximum elements $i_j$ of each column, we have:
\begin{align*}
    \|\widehat{Z}_{t} - Z_{t}\|_{\infty,1} \leq {} & \sum_{j=1}^d \sqrt{\frac{2\text{Var}\big[\,\sum_{h=1}^H \overline{V}_{t,h}[i_j]\phi(s_h,a_h)[j]\,\big]\log \tfrac{2dK}{\delta'}}{n}} + \frac{2dH\log \tfrac{2dK}{\delta'}}{n}\\
    = {} & \sum_{j=1}^d  \sqrt{\frac{2\E\big[\,\big(\sum_{h=1}^H \phi(s_h,a_h)[j]\big)^2\,\big]\log \tfrac{2dK}{\delta'}}{n}} + \frac{2dH\log \tfrac{2dK}{\delta'}}{n} & \\
    \leq {} & \sqrt{\frac{2d\E\big[\,\sum_{j=1}^d  \big(\sum_{h=1}^H \phi(s_h,a_h)[j]\big)^2\,\big]\log \tfrac{2dK}{\delta'}}{n}} + \frac{2dH\log \tfrac{2dK}{\delta'}}{n} & \\
    \leq {} & \sqrt{\frac{2d\E\big[\, \big(\sum_{j=1}^d \sum_{h=1}^H \phi(s_h,a_h)[j]\big)^2\,\big]\log \tfrac{2dK}{\delta'}}{n}} + \frac{2dH\log \tfrac{2dK}{\delta'}}{n} & \\
    \leq {} & H\sqrt{\frac{2d\log \tfrac{2dK}{\delta'}}{n}} + \frac{2dH\log \tfrac{2dK}{\delta'}}{n} &
\end{align*}
Here, for the first step, we have used the property that $\overline{V}_{t,h}[i_j] \leq 1$, the fact that variance is bounded by the second moment. The next step can be obtained by using Cauchy-Schwartz inequality. The last second step uses that property that for non-negative $a_j$, $\sum_{j} a_j^2 \leq (\sum_j a_j)^2$. Now, if $\frac{2d\log \frac{2}{\delta'}}{n} \leq 1$, the above is bounded by $2\sqrt{\frac{2d\log \frac{2}{\delta'}}{n}}$. 

Therefore, summing up the two terms with $\delta'=\delta_2/2$, with probability at least $1-\delta_2$ and for all $W \in \C{W}$, we have:
\begin{align*}
   \Big|\hat{y}_t - \sum_{h=1}^H\E_{d^{W_t}_{M^*,h}}\big[\E_{M^*} \big[r_h + V_{t,h+1}(s_{h+1}) | s_h,a_h\big]\Big| + \Big|\langle W, Z_{t} - \widehat{Z}_t \rangle\Big| \leq {} & \sqrt{\frac{8dH^2\log \frac{4dK}{\delta_2}}{n}} + \sqrt{\frac{H^2\log \frac{4}{\delta_2}}{2n}}
\end{align*}
With some algebra, it can be verified that setting $n = \frac{1800d^2KH^2}{\epsilon^2} \log \frac{4dK}{\delta_2}$ makes the total error bounded by $\frac{\epsilon}{12\sqrt{dK}}$ with failure probability $\delta_2$.
\end{proof}

\subsection{Proof of Theorem~\ref{thm:mainbound}}
\label{sec:mainproof}
\begin{proof}
With the key lemmas in Section~\ref{sec:keylem}, we can now prove the main result. For the main theorem, we need to ensure that the requirements in lemma~\ref{lem:explore-exploit} and \ref{lem:iteration-complexity} are satisfied. Since, our method maintains a version space of plausible weights, the validity of each iteration depends on every previous iteration being valid. Therefore, for a total failure probability of $\delta$ for the algorithm, we assume:

\begin{enumerate}[label=(\roman*)]
    \item Estimation of $\hat{\C{E}}(W_t)$ for all iterations: total failure probability $\delta/2$
    \item Updating the version space $\C{W}_t$: total failure probability $\delta/2$
\end{enumerate}
We set $\delta_1 = \delta/2T$ and $\delta_2 = \delta/2T$ in Lemmas~\ref{lem:model-misfit-conc} and \ref{lem:dotp-conc-linear} respectively with $T = dK \log \frac{2\sqrt{2K}H}{\epsilon} / \log \frac{5}{3}$. By taking a union bound over maximum number of iterations, the total failure probability is bounded by $\delta$. Thus, the total number of trajectories unrolled by the algorithm is:
\begin{align*}
    T (n_\text{eval} + n) \leq {} & \Big( dK \log \frac{2\sqrt{2K}H}{\epsilon} / \log \frac{5}{3} \Big) \Big( \frac{8}{\epsilon^2} \log \frac{4T}{\delta} + \frac{1800d^2KH^2}{\epsilon^2} \log \frac{8dKT}{\delta} \Big) \\
    = {} & \tilde{\C{O}}\Big(\frac{d^3K^2H^2}{\epsilon^2} \log \frac{dKH}{\delta}\Big)
\end{align*}
Therefore, by the termination guarantee in Lemma~\ref{lem:explore-exploit}, we arrive at the desired upper bound on the number of trajectories required to guarantee a policy with value $v^\pi_{M^*} \ge v^* - (3\sqrt{dK}+1)H\theta$.
\end{proof}

%\begin{lemma}[Simulation Lemma \citep{kearns2002near, modi2018markov}]
%\label{lem:sim-lemma}
%Let $M(W)$ and $M^*$ be two MDPs with the same state-action space. If the transition dynamics and reward functions of the two MDPs are such that:
%\begin{align*}
%    \|{P}^W(\cdot|s,a) - {P}^*(\cdot|s,a)\|_1 \le {} & \epsilon/2H \qquad \forall s\ \in \Scal, a \in \Acal\\
%    |{R}^W(\cdot|s,a) - {R}^*(\cdot|s,a)| \le {} & \epsilon/2 \qquad \forall s\ \in \Scal, a \in \Acal
%\end{align*}
%then, for every policy $\pi$, we have:
%\begin{align}
%    |V^{\pi}_W - V^{\pi}_{W^*}| \le \epsilon
%\end{align}
%\end{lemma}

\section{The Implication of  Proposition~\ref{prop:unstructured_feature_selection} on the Hardness of Learning State Abstractions} \label{app:abstraction}
Here we show that the proof of Proposition~\ref{prop:unstructured_feature_selection} can be adapted to show a related hardness result for learning with state abstractions. A state abstraction is a mapping $\phi$ that maps the raw state space $\Scal$ to some finite abstract state space $\Scal_\phi$, typically much smaller in size. When an abstraction $\phi$ with good properties (e.g., preserving reward and transition dynamics) is known, one can leverage it in exploration and obtain a sample complexity that is polynomial in $|\Scal_\phi|$ instead of $|\Scal|$. Among different types of abstractions, \emph{bisimulation} \citep{whitt1978approximations, givan2003equivalence} is a very strict notion that comes with many nice properties \citep{li2006towards}. 

An open problem in state abstraction literature has been whether it is possible to perform model selection over a large set of candidate abstractions, i.e., designing an algorithm whose sample complexity only scales sublinearly (or ideally, logarithmically) with the cardinality of the candidate set. Using the construction from Proposition~\ref{prop:unstructured_feature_selection}, we show that this is impossible without further assumptions:
\begin{proposition}
Consider a learner in an MDP that is equipped with a set of state abstractions, $\{\phi_1, \phi_2, \ldots, \phi_N\}$. Each abstraction $\phi_i$ maps the raw state space $\Scal$ to a finite abstract state space $\Scal_{\phi_i}$. Even if there exists $i^* \in [N]$ such that $\phi^* = \phi_i$ is a bisimulation, no algorithm can achieve poly$(|\Scal_{\phi^*}|, |\Acal|, H, 1/\epsilon, 1/\delta, N^{1-\alpha})$ sample complexity for any $\alpha >0$.
\end{proposition}
\begin{proof}
Following the proof of Proposition~\ref{prop:unstructured_feature_selection}, we consider the family of MDPs that share the same deterministic transition dynamics with a complete tree structure, where each MDP only has one rewarding leaf. Let $N=2^H$ be the number of leaves, and let the MDPs be $\{M_i\}$ where the index indicates the rewarding leaf. We then construct a set of $N$ abstractions where one of them will always be a bisimulation regardless of which MDP we choose from the family. Consider the $i$-th abstraction, $\phi_i$. At each level $h$, $\phi_i$ aggregates the state on the optimal path in its own equivalence class, and aggregates all other states together. It does not aggregate states across levels. It is easy to verify that $\phi_i$ is a bisimulation in MDP $M_i$, and $|\Scal_{\phi_i}| = 2H$. If the hypothetical algorithm existed, it would achieve a sample complexity sublinear in $N$, as $N=2^H$ and all other relevant parameters (e.g., $|\Scal_{\phi_i}|$ and $H$) are at most polynomial in $H$. However, the set of abstractions $\{\phi_i\}$ is uninformative in this specific problem and does not affect the $\Omega(2^H)$ sample complexity, which completes the proof.
\end{proof}

\end{document}